\def\year{2022}\relax
\documentclass[letterpaper]{article} 
\pdfoutput=1
\usepackage{aaai22}  
\usepackage{times}  
\usepackage{helvet}  
\usepackage{courier}  
\usepackage[hyphens]{url}  
\usepackage{graphicx} 
\usepackage{natbib}  
\usepackage{caption} 
\DeclareCaptionStyle{ruled}{labelfont=normalfont,labelsep=colon,strut=off} 
\usepackage{algorithm}
\usepackage{algorithmic}

\usepackage{newfloat}
\usepackage{listings}
\floatstyle{ruled}
\newfloat{listing}{tb}{lst}{}
\floatname{listing}{Listing}
\usepackage{custom}

\title{Preemptive Image Robustification for Protecting Users \\ against Man-in-the-Middle Adversarial Attacks}
\author {
    Seungyong Moon\equalcontrib\textsuperscript{\rm 1},
    Gaon An\equalcontrib\textsuperscript{\rm 1},
    Hyun Oh Song\thanks{Corresponding author}\textsuperscript{\rm 1}\textsuperscript{\rm 2}
}
\affiliations {
    \textsuperscript{\rm 1} Department of Computer Science and Engineering, Seoul National University, Seoul, Korea \\
    \textsuperscript{\rm 2} DeepMetrics, Seoul, Korea \\
    \{symoon11, white0234, hyunoh\}@mllab.snu.ac.kr
}

\begin{document}

\maketitle

\begin{abstract}
Deep neural networks have become the driving force of modern image recognition systems. However, the vulnerability of neural networks against adversarial attacks poses a serious threat to the people affected by these systems. In this paper, we focus on a real-world threat model where a Man-in-the-Middle adversary maliciously intercepts and perturbs images web users upload online. This type of attack can raise severe ethical concerns on top of simple performance degradation. To prevent this attack, we devise a novel bi-level optimization algorithm that finds points in the vicinity of natural images that are robust to adversarial perturbations. Experiments on CIFAR-10 and ImageNet show our method can effectively robustify natural images within the given modification budget. We also show the proposed method can improve robustness when jointly used with randomized smoothing.
\end{abstract}

\section{Introduction}\label{sec:intro}

Recent progress in deep neural networks has enabled substantial performance gains in various computer vision tasks, including image classification, object detection, and semantic segmentation. Leveraging this advance, more practitioners are deploying neural network-based image recognition systems in real-world applications, such as image tagging or face recognition. However, neural networks are vulnerable to \emph{adversarial examples} \citep{szegedy13}, minute input perturbations intentionally designed to mislead networks to yield incorrect predictions. These adversarial examples can significantly degrade the performance of the network models, raising security concerns about their deployment.

\begin{figure}[t]
	\centering
	\includegraphics[width=0.86\columnwidth]{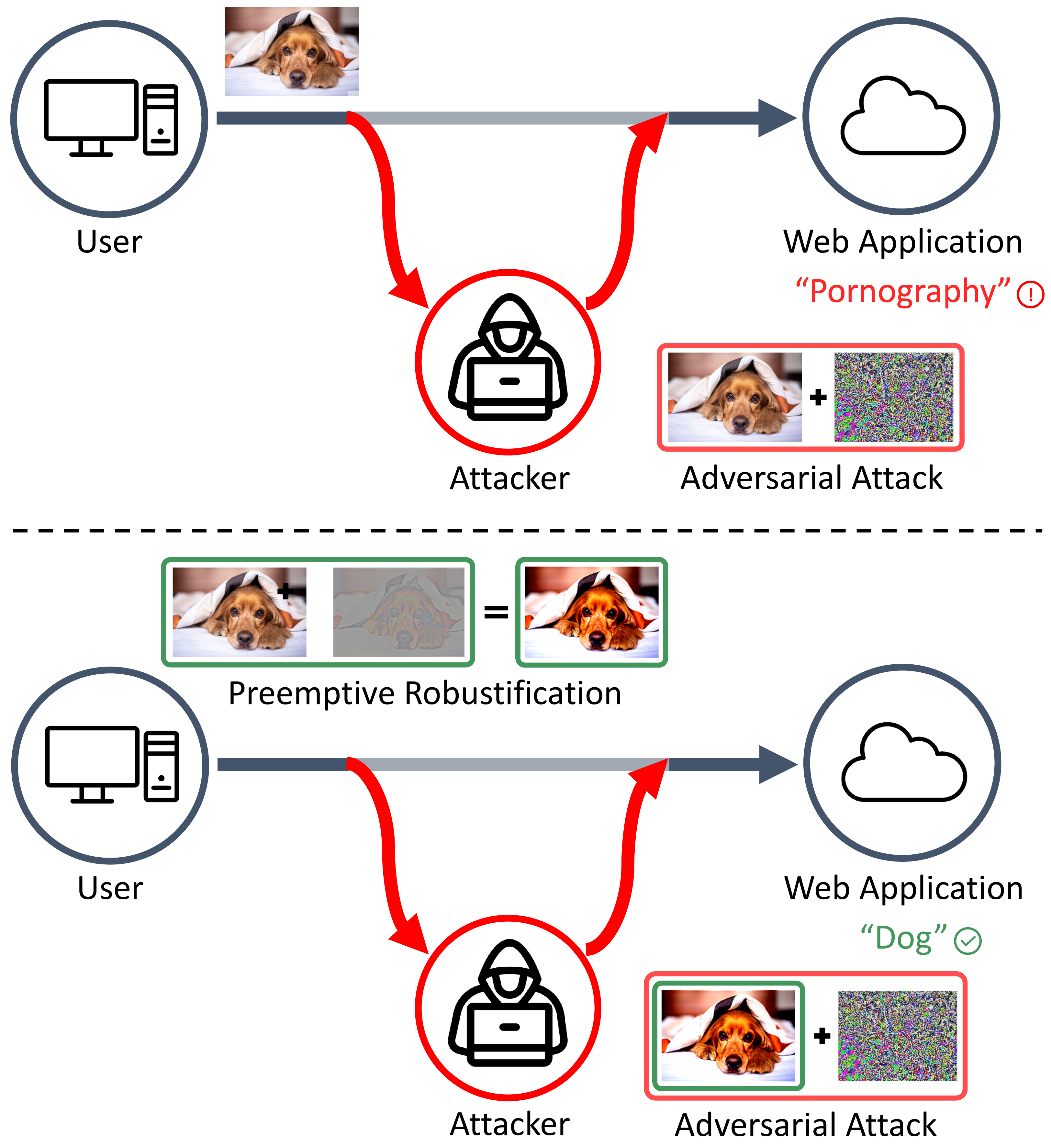}
	\caption{Illustration of our proposed method. Without protection, a MitM adversary can easily perturb the user's image to be misclassified by the web application (top). Our proposed method preemptively robustifies the user's image, securing the image from adversarial attack (bottom).}
	\label{fig:overview}
\end{figure}

When these image recognition systems are deployed to applications where users freely upload images from local machines to remote storage, such as social media, this vulnerability can pose another serious threat, especially to the individual application users. Consider there exists a man-in-the-middle (MitM) adversary that can intercept and add perturbations to the images web users upload during transmission (Figure 1). Then, this adversary can easily vandalize neural network-based web services such as image auto-tagging in social media apps by perturbing the images to be misclassified. This type of attack can severely deteriorate user experience, especially as the adversary can further use this attack to insult the uploaders beyond simple misclassification. Even though the MitM attack is one of the most lethal cyber threats \citep{desmedt2011man, li19video, wang19maliciousgen}, protecting neural networks from this type of attack has been much less studied in adversarial machine learning literature.

In this work, we develop a new defense framework to protect web users from MitM attacks. To provide more effective protection measures for the users, we focus on the fact that users hold control of their images before the adversary, unlike in conventional adversarial attack scenarios. Based on this observation, we ask the following question:

\begin{itemize}
	\item \emph{Can we preemptively manipulate images slightly to be robust against MitM adversarial attacks?}
\end{itemize}

To answer this, we explore the existence of points in image space that are resistant to adversarial perturbations, given a trained classifier. We propose a novel bi-level optimization algorithm for finding those robust points under a given modification budget starting from natural images and measure the degree of robustness achievable by utilizing these points, which we denote as \emph{preemptive robustness}. Moreover, we propose a new network training scheme that further improves preemptive robustness. We validate the effectiveness of our proposed framework in the image classification task on CIFAR-10 \citep{cifar} and ImageNet \citep{imagenet} datasets. Our extensive experiments demonstrate that our framework can successfully robustify images against a wide range of adversarial attacks in both black-box and white-box settings. We also observe that our method can enhance the preemptive robustness on smoothed classifiers.

In summary, our main contributions are as follows:

\begin{itemize}
	\item We introduce a novel real-world adversarial attack scenario targeting users in image recognition systems.
	\item We propose a new defense framework to improve preemptive robustness and formulate the preemptive robustification process as a bi-level optimization problem.
	\item We demonstrate our proposed framework can significantly improve preemptive robustness against a wide range of adversarial attacks on standard benchmarks.
\end{itemize}

\section{Related Works}\label{sec:rel}

\paragraph{Adversarial robustness of neural networks} 
Most prior work on adversarial robustness aims to train neural networks that achieve high accuracies on adversarially perturbed inputs. PGD adversarial training improves the empirical robustness of neural networks by augmenting training data with multi-step PGD adversarial examples \citep{pgd}. Some recent works report performance gains over PGD adversarial training by modifying the adversarial example generation procedure \citep{zhang2019defense, zhang2020fat}. However, most of the recent algorithmic improvements can be matched by simply using early stopping with PGD adversarial training \citep{rice20, croce2020reliable}. Despite its effectiveness, a major drawback of adversarial training is that it takes a huge computational cost to generate adversarial examples during training. To address this, several works develop fast adversarial training methods by reusing the gradient computation or reducing the number of attack iterations \citep{free, yopo, fast}.

Although such adversarial training methods can significantly improve the empirical robustness of neural networks, there is no guarantee that a stronger, newly-discovered attack would not break them. To address this, a separate line of work focuses on certifying robustness against any adversarial perturbations \cite{raghunathan2018certified, wong2018provable} but often has difficulty in scaling to large neural networks. Randomized smoothing, a method that injects random additive noises to inputs to construct smoothed classifiers from base classifiers, has been considered the most successful certified defense approach that can be applied to large neural networks \citep{lecuyer19, cohen19, salman19, yang2020randomized}.

\paragraph{Preemptive image manipulation for robustness}
There have been a few studies on preemptive image manipulation to protect images from being exploited, yet most of them utilize it for privacy protection. In the facial recognition task, prior work proposes an algorithm that slightly modifies personal photos before uploading them to social media to make them hard to be identified by malicious person recognition systems \citep{oh2017adversarial, shan2020fawkes, cherepanova2021lowkey}. Our work differs from this prior work in that our goal is to robustify images to be correctly identified by classification systems even under adversarial perturbations. 

The most relevant to our work is an approach of \citet{salman2020unadversarial}, which develops patches that can boost robustness to common corruptions when applied to clean images. However, we consider the problem of manipulating images to be correctly classified against worst-case adversarial perturbations, which are artificially designed to cause misclassification. Ensuring robustness to adversarial perturbations is much more challenging than robustifying images against common corruptions, which are not the worst-case perturbations.

\section{Methods}\label{sec:method}

\subsection{Problem Setup}\label{sec:setup}

To start, we introduce our defense framework for image classification models, along with the adversarial threat model. In our framework, a defender can preemptively modify the original image $x_o$ to produce a new image $x_r$ that is visually indistinguishable from $x_o$ ahead of adversarial attacks. After the modification, the defender discards the original image $x_o$, so that adversary can only see the modified image $x_r$. Under this framework, we can consider two types of adversaries:
\begin{itemize}
	\item A \emph{grey-box} adversary who has complete knowledge of the classification model but does not know the modification algorithm exists. 
	\item A \emph{white-box} adversary who not only has full access to the classification model but also is aware of the existence of the modification algorithm and how it works.
\end{itemize}
The grey-box adversary will regard the given image $x_r$ as the original image and attempt to find an adversarial example near $x_r$, as from the conventional adversarial literature. In contrast, the white-box adversary recognizes that $x_r$ is a modified version. Thus, the white-box adversary will instead try to guess the location of the original image $x_o$ and craft an adversarial example near it.

In this paper, we investigate the defender's optimal strategy for manipulating original images to be resistant against these two adversaries. First, we develop an algorithm that preemptively robustifies original images against the grey-box adversary. Then, we demonstrate our proposed algorithm also exhibits high robustness against adaptively designed white-box attacks.

\subsection{Preemptive Robustness}\label{sec:definition}

We now formally introduce the concept of \emph{preemptive robustness}. We begin by recalling the definition of adversarial examples. Let $c: \mathcal{X} \to \mathcal{Y}$ be a classifier which maps images to class labels. Given an original image $x_o \in \mathcal{X}$ and its class $y_o \in \mathcal{Y}$, suppose $x_o$ is correctly classified. Then, an adversarial example $x_o^a \in \mathcal{X}$ of $x_o$ is defined as an image in the neighborhood of $x_o$ such that the classifier changes its prediction, \ie, $c(x_o^a) \neq c(x_o)$ and $x_o^a \in B_{\epsilon}(x_o)$. Here, $\epsilon > 0$ is the perturbation budget of the adversary and $B_{\epsilon}(x) = \{ x' \in \mathcal{X} \mid \| x' - x \|_p \le \epsilon \}$ denotes the closed $\ell_p$-ball of radius $\epsilon$ centered at $x$. Throughout this paper, we consider $p \in \{2, \infty\}$, the most common settings in adversarial machine learning literature. If the classifier gives robust predictions in the neighborhood of $x_o$, then we say $x_o$ is robust against adversarial perturbations.

We can extend this notion of adversarial robustness to the whole image space $\mathcal{X}$. To do this, we define the \emph{robust region} of a classifier $c$ as the set of images that $c$ can output robust predictions in the presence of adversarial perturbations.
\begin{definition}[$\epsilon$-robust region]
	Let $c: \mathcal{X} \to \mathcal{Y}$ be a classifier and $\epsilon > 0$ be the perturbation budget of an adversary. The $\epsilon$-robust region of the classifier $c$ is defined by $ R_{\epsilon} (c) \coloneqq \{ x \in \mathcal{X} \mid c(x') = c(x), ~\forall x' \in B_{\epsilon}(x) \}$.
\end{definition}

Now, consider a defender who can preemptively manipulate $x_o$ under a small modification budget $\delta > 0$ to generate a new image $x_r \in B_{\delta}(x_o)$, and a grey-box adversary who aims to find an adversarial example near $x_r$. Then, the defender's optimal strategy against the adversary is to make $x_r$ be correctly classified as $y_o$ and locate in the robust region $R_{\epsilon}(c)$ so that $x_r$ is robust to adversarial perturbations. If both of these two conditions are satisfied, we say $x_o$ is \emph{preemptively robust} against the grey-box adversary and $x_r$ is a \emph{preemptively robustified image} of $x_o$.
\begin{definition}[Preemptive robustness, grey-box]
	Let $c: \mathcal{X} \to \mathcal{Y}$ be a classifier and $\delta, \epsilon > 0$ be the modification budgets of the defender and the grey-box adversary, respectively. An original image $x_o$ with its class $y_o$ is preemptively robust against the grey-box adversary if there exists $x_r \in B_\delta(x_o)$ such that (\lowercase\expandafter{\romannumeral1}) $c(x_r) = y_o$ and (\lowercase\expandafter{\romannumeral2}) $x_r \in R_{\epsilon}(c)$.
	\label{def:preemptive}
\end{definition}

Next, we consider a white-box adversary against the defender. Let us denote the manipulation algorithm of the defender by $m: \mathcal{X} \times \mathcal{Y} \to \mathcal{X}$, \ie, $x_r = m(x_o, y_o)$. Then, the white-box adversary will adaptively design its attack algorithm $a_m: \mathcal{X} \times \mathcal{Y} \to \mathcal{X}$, which takes $x_r$ and $y_o$ as inputs and produces a candidate of the adversarial example. For the output $a_m(x_r, y_o)$ to be a valid adversarial example, it should be misclassified and located in $B_{\epsilon}(x_o)$. If the output is not a valid adversarial example, we say $x_o$ is \emph{preemptively robust} against the white-box adversary.
\begin{definition}[Preemptive robustness, white-box]
	Let $m: \mathcal{X} \times \mathcal{Y} \to \mathcal{X}$ be the defender's manipulation algorithm and $a_m: \mathcal{X} \times \mathcal{Y} \to \mathcal{X}$ be the adaptive attack algorithm of the white-box adversary. Given an original image $x_o$ and its class $y_o$, let $x_r = m(x_o, y_o)$ denote the resulting image of the defender's algorithm. Then, $x_o$ is called preemptively robust against the white-box adversary if either of the following conditions is satisfied: (\lowercase\expandafter{\romannumeral1}) $c(a_m(x_r, y_o)) = y_o$ or (\lowercase\expandafter{\romannumeral2}) $a_m(x_r,  y_o) \notin B_{\epsilon}(x_o)$. 
\end{definition}

Note that since the white-box adversary does not have any information about the original $x_o$ in the MitM setting, forcing $a_m(x_r, y_o)$ to lie in $B_\epsilon(x_o)$ is a non-trivial task for the adversary.

\subsection{Preemptive Robustification Algorithm}\label{sec:safe_spot}

In this subsection, we develop an algorithm for preemptively robustifying original images against the grey-box adversary. Given a classifier $c$, finding a preemptively robustified image $x_r$ from an original image $x_o$ can be formulated as the following optimization problem, which is directly from \Cref{def:preemptive}:
\begin{align*}
	& \minimize_{x_r} ~~\: \mathds{1}_{c(x_r) \neq y_o} + \mathds{1}_{x_r \notin R_{\epsilon}(c)} \\
	& \: \mathrm{subject~to} ~~ \|x_r - x_{o}\|_{p} \le \delta,
\end{align*}
where $\mathds{1}$ is the 0-1 loss function. 

Note that in this formulation, the defender requires the ground-truth label $y_o$. However, images in real-world applications are usually unlabeled unless users manually annotate their images. Therefore, it is natural to assume that the defender does not have access to the ground-truth label $y_o$. In this case, we utilize the classifier's prediction $c(x_o)$ instead of $y_o$:
\begin{align*}
	& \minimize_{x_r} ~~\: \mathds{1}_{c(x_r) \neq c(x_o)} + \mathds{1}_{x_r \notin R_{\epsilon}(c)} \\
	& \: \mathrm{subject~to} ~~ \|x_r - x_o\|_{p} \le \delta.
\end{align*}
As $x_r \notin R_{\epsilon}(c)$ implies there exists an adversarial example $x_r^a \in B_{\epsilon}(x_r)$ such that $c(x_r^a) \neq c(x_r)$, we can reformulate the optimization problem as
\begin{align*}
	& \minimize_{x_r} ~~\: \mathds{1}_{c(x_r) \neq c(x_o)} + \sup_{x_r^a} ~ \mathds{1}_{c(x_r^a) \neq c(x_r)} \\
	& \: \mathrm{subject~to} ~~ \|x_r - x_o\|_{p} \le \delta ~~ \mathrm{and} ~~ \|x_r^a - x_r\|_{p} \le \epsilon.
\end{align*}
Since 0-1 loss is not differentiable, we use the cross-entropy loss $\ell: \mathcal{X} \times \mathcal{Y} \to \mathbb{R}^{+}$ of the classifier $c$ as the convex surrogate loss function:
\begin{align}
	\label{eqn:safe_spot_naive}
	& \minimize_{x_r} ~~\: \ell(x_r, c(x_o)) + \sup_{x_r^a} ~ \ell(x_r^a, c(x_r)) \\
	& \: \mathrm{subject~to} ~~ \|x_r - x_o\|_{p} \le \delta ~~ \mathrm{and} ~~ \|x_r^a - x_r\|_{p} \le \epsilon \nonumber.
\end{align}

Let $h(x_r)$ denote the objective in \Cref{eqn:safe_spot_naive}. Instead of minimizing $h(x_r)$ directly, we minimize $\tilde{h}(x_r) = \supx_{x_r^a \in B_{\epsilon}(x_r)} \ell(x_r^a, c(x_o))$ since it upper bounds $h(x_r)$ when sufficiently minimized due to \Cref{lem:upper}.
\begin{lemma}
	\label{lem:upper}
	If $\tilde{h}(x_r) \le -\log(0.5) \simeq 0.6931$, then $h(x_r) \le 2 \tilde{h}(x_r)$.
\end{lemma}
\begin{proof}
	See Supplementary A.1.
\end{proof}
Finally, we have the following optimization problem:
\begin{align}
	\label{eqn:safe_spot}
	& \minimize_{x_r} ~~\: \sup_{x_r^a} ~ \ell(x_r^a, c(x_o)) \\
	& \: \mathrm{subject~to} ~~ \|x_r - x_o\|_p \le \delta ~~ \mathrm{and} ~~ \|x_r^a - x_r\|_p \le \epsilon \nonumber.
\end{align}

\begin{algorithm}[t]
	\caption{Preemptive robustification algorithm}
	\label{alg:safe_spot}
	\begin{algorithmic}[0]
		\INPUT Original image and its prediction $(x_{o}, c(x_{o}))$
		\STATE $x_r = x_o$ \textit{// or randomly initialized in $B_\delta(x_o)$}
		\FOR{$i=1, \ldots, \text{MAXITER}$}
		\STATE \textit{// Generate $N$ PGD adversarial examples}
		\FOR{$n=1, \ldots, N$}
		\STATE $x_{r, n}^a = x_r + \eta$ where $\eta \sim \Ucal(B_\epsilon(0))$
		\FOR{$t=1, \ldots, T$}
		\STATE$ x_{r, n}^a \leftarrow \Pi_{x_r, \epsilon} \left( f(x_{r, n}^a; c(x_o), \ell) \right)$
		\ENDFOR
		\ENDFOR
		\STATE \textit{// Update image} \\
		$x_r \leftarrow \Pi_{x_{o}, \delta} \left( x_r - \beta \cdot \dfrac{1}{N}\sum\limits_{n=1}^N \dfrac{\partial \ell(x_{r, n}^a, c(x_o))}{\partial x_r} \right)$
		\ENDFOR
		\OUTPUT $x_r$
	\end{algorithmic}
\end{algorithm}

To solve \Cref{eqn:safe_spot}, we first approximate the inner maximization problem by running $T$-step PGD \citep{pgd} whose dynamics is given by
\begin{align*}
	x_r^{a, 0} &= x_r + \eta \tag{random start} \\
	\tilde{x}_r^{a, t} &= f\left(x_r^{a, t-1}; c(x_{o}), \ell\right) \tag{adversarial update} \\
	x_r^{a, t} &= \Pi_{x_r, \epsilon} \left( \tilde{x}_r^{a, t} \right), \tag{projection}
\end{align*}
where $\eta$ is a noise uniformly sampled from $B_{\epsilon}(0)$, $f$ is FGSM \citep{fgsm} defined as
\begin{align*}
	f(x; y, \ell) = \begin{dcases} 
		x + \alpha \cdot \sgn \left( \nabla_x \ell(x, y) \right) & \text{if} ~~ p=\infty \\
		x + \alpha \cdot \frac{\nabla_x \ell(x, y)}{\left\| \nabla_x \ell(x, y) \right\|_2} & \text{if} ~~ p=2,
	\end{dcases}
\end{align*}
and $\Pi_{x_r, \epsilon}$ is a projection operation onto $B_{\epsilon}(x_r)$. Then, we iteratively solve the approximate problem given by replacing $x_r^a$ to $x_r^{a, T}$ in \Cref{eqn:safe_spot}. To update $x_r$, we compute the gradient of $\ell(x_r^a, c(x_o))$ with respect to $x_r$ expressed as
\begin{align*}
	& \frac{\partial \ell\left(x_r^a, c(x_{o})\right)}{\partial x_r} = \\
	& \quad ~~ \nabla_x \tilde{f}\left(x_r^{a, 0}\right)^\intercal \cdot \cdots \cdot \nabla_x \tilde{f}\left(x_r^{a, T-1}\right)^\intercal \cdot \nabla_x \ell\left(x_r^{a, T}, c(x_o)\right),
\end{align*}
where $\nabla_x \tilde{f}$ is the Jacobian matrix of $ \tilde{f} = \Pi_{x_r, \epsilon} \circ f$ which can be computed via back-propagation. After computing the gradient, we update $x_r$ by projected gradient descent method:
\begin{align*}
	x_r \leftarrow \Pi_{x_o, \delta} \left( x_r - \beta \cdot \frac{\partial \ell(x_r^a, c(x_o))}{\partial x_r} \right).
\end{align*}
Note that $\ell(x_r^a, c(x_{o}))$ is a random variable dependent on $\eta$. Therefore, we generate $N$ adversarial examples $\{x_{r, n}^a\}_{n=1}^N$ with different noises and optimize the sample mean of the losses instead. \Cref{alg:safe_spot} shows the overall preemptive robustification algorithm and \Cref{fig:safe_spot_finding} illustrates the optimization process. Some examples of robustified images generated from our algorithm are shown in Supplementary D.

\begin{figure}[t]
	\centering
	\includegraphics[width=0.95\columnwidth]{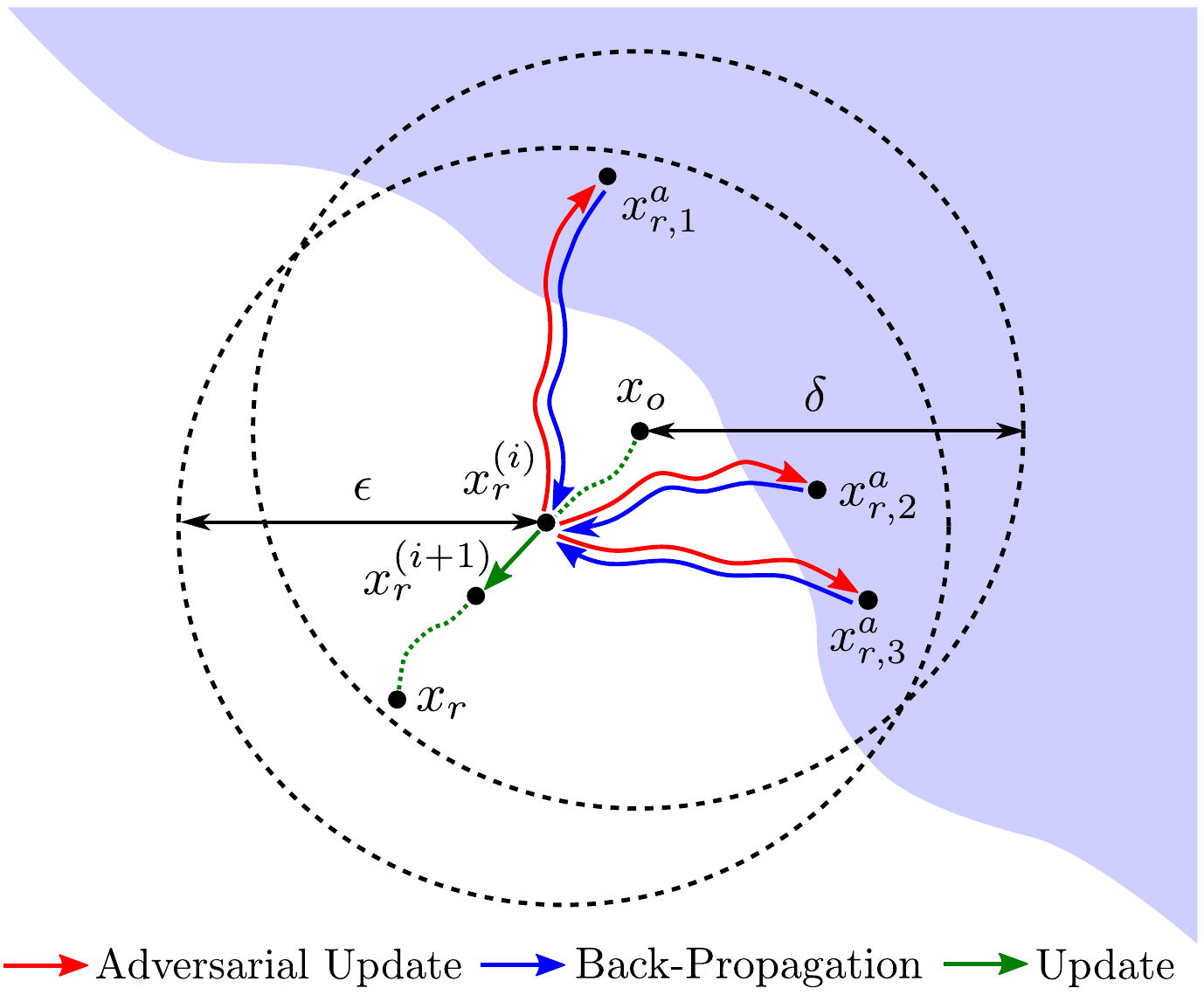}
	\caption{Illustration of the preemptive robustification process. The shaded region represents the set of misclassified points.}
	\label{fig:safe_spot_finding}
\end{figure}

\subsection{Computing Update Gradient without Second-Order Derivatives}
Computing the update gradient with respect to $x_r$ involves the use of second-order derivatives of the loss function $\ell$ since the dynamics $f$ contains the loss gradient $\nabla_x \ell$. Standard deep learning libraries, such as PyTorch \citep{pytorch}, support the computation of these higher-order derivatives. However, it imposes a huge memory burden as the size of the computational graph increases. Furthermore, when $p = 2$, computing the update gradient with the second-order derivatives might cause an \emph{exploding gradient problem} if the loss gradient vanishes by \Cref{lemma:jacobian} and \Cref{prop:jacobian}.

\begin{lemma}
	\label{lemma:jacobian}
	Suppose $\ell$ is twice-differentiable and its second partial derivatives are continuous. If $p=2$, the Jacobian of the dynamics $f$ is
	\begin{align*} 
		\nabla_x f = I + \alpha \cdot \left(I - \left(\frac{g}{\| g \|_2}\right)\left(\frac{g}{\| g \|_2}\right)^\intercal \right) \frac{H}{\| g \|_2},
	\end{align*}
	where $g = \nabla_{x} \ell$ and $H = \nabla_{x}^2 \ell$.
\end{lemma}
\begin{proof} 
	See Supplementary B.1.
\end{proof}

\begin{proposition}
	\label{prop:jacobian}
	If the maximum eigenvalue of $H$ in absolute value is $\sigma$, then 
	\begin{align*}
		\left\| \nabla_x \tilde{f}^\intercal \cdot a \right\|_2 \le \left( 1 + \alpha \cdot \dfrac{\sigma}{\| g \|_2} \right) \| a \|_2.
	\end{align*}
\end{proposition}
\begin{proof}
	See Supplementary B.2.
\end{proof}

As we update $x_r$, the loss gradients $g$ of $x_r$ and its intermediate adversarial examples $x_r^{a, t}$ reduce to zero, which might cause the update gradient to explode and destabilize the update process. To address this problem, we approximate the update gradient by excluding the second-order derivatives, following the practice in \citet{maml}. We also include an experiment in comparison to using the exact update gradient in Supplementary B.3. For the case of $p=\infty$, the second-order derivatives naturally vanish since we take the sign of the loss gradient $\nabla_x \ell(x, y)$. Therefore, the approximate gradient is equal to the exact update gradient.

\subsection{Network Training Scheme for Improving Preemptive Robustness}\label{sec:safe_adv_train}

So far, we have explored how the defender can preemptively robustify the original image, given a pre-trained classifier. Now, we explore the defender's network training scheme for a classifier where data points are preemptively robust with high probability. Suppose the defender has a labeled training set, which is drawn from a true data distribution $\Dcal$. To induce data points to be preemptively robust, the defender's optimal training objective should have the following form:
\begin{align*}
	&\minimize_{\theta} ~~\: \mathop{\mathbb{E}}_{(x_o, y_o) \sim \Dcal} ~ \left[ \ell(\hat{x}_r^a, y_o; \theta) \right] \\
	&\: \mathrm{subject~to}  ~~~~~~~~ \hat{x}_r^a = \argmax_{x_r^a \in B_{\epsilon}(\hat{x}_r)} ~ \ell (x_r^a, y_o) \\
	&\: \qquad\qquad\qquad ~\,  \hat{x}_r = \argmin_{x_r \in B_{\delta}(x_o)} ~ \sup_{x_r^a \in B_{\epsilon}(x_r)} ~ \ell (x_r^a, y_o),
\end{align*}
where $\theta$ is the set of trainable parameters. Concretely, the defender first attempts to craft a candidate for preemptively robustified points $\hat{x}_r$ of the original data point $x_o$. Then, the defender generates an adversarial example $\hat{x}_r^a$ of $\hat{x}_r$ and minimizes its cross-entropy loss $\ell(\hat{x}_r^a, y_o; \theta)$ so that $\hat{x}_r$ becomes resistant to adversarial perturbations. Note that the ground-truth label $y_o$ is used instead of the prediction $c(x_o)$, since we assume the ground-truth label of the training set is given.

The most direct way to optimize the objective would be to find $\hat{x}_r$ from $x_o$ using our preemptive robustification algorithm and perform $K$-step PGD adversarial training \citep{pgd} with $\hat{x}_r$. However, since our algorithm requires running $T$-step PGD dynamics per each update, the proposed training procedure would be more computationally demanding than standard PGD adversarial training. To ease this problem, we replace the inner maximization $\sup_{x_r^a} \ell(x_r^a, y_o)$ in the preemptive robustification process by $\ell(x_r, y_o)$:
\begin{align*}
	\hat{x}_r & = \argmin_{x_r \in B_{\delta}(x_o)} ~ \sup_{x_r^a \in B_{\epsilon}(x_r)} ~ \ell (x_r^a, y_o) \\
	\Longrightarrow ~ \hat{x}_r & = \argmin_{x_r \in B_{\delta}(x_o)} ~ \ell (x_r, y_o).
\end{align*}

Then, $\hat{x}_r$ can be easily computed by running $L$-step PGD on $x_o$ towards minimizing the cross-entropy loss. We denote this training scheme as \emph{preemptively robust training}. The full training procedure is summarized in \Cref{alg:safe_spot_aware} (differences with the standard adversarial training marked in blue).

Note that the standard adversarial training is a specific case of our preemptively robust training, forcing training data $x_o$ to be far from the decision boundary. However, recent work demonstrates there is a trade-off between the classification error and the boundary error, which is why standard adversarial training significantly decreases the clean accuracy \cite{odds, trades}. In contrast, our proposed training scheme allows original images $x_o$ to lie near the decision boundary and only enforces the preemptively robustified images $\hat{x}_r$ to be distant from the boundary. Our experiments in \Cref{sec:exp} show preemptively robust training is less prone to suffer from the clean accuracy drop due to this flexibility.

\begin{algorithm}[t]
	\caption{Preemptively robust training, $p=\infty$}
	\label{alg:safe_spot_aware}
	\begin{algorithmic}[0]
		\INPUT Training dataset $\mathcal{D}_{train}$, maximum epoch $N$
		\FOR{$n = 1, \ldots, N$}
		\FOR{$(x_o, y_o) \in \mathcal{D}_{train}$}
		\STATE \blue{\textit{// Do $L$-step PGD towards minimizing loss}}\\
		\blue{
		$x_r = x_o + \eta$ where $\eta \sim B_{\delta}(0)$
		\FOR{$l=1, \ldots, L$}
		\STATE $x_r \leftarrow \Pi_{x_o, \delta} \left( x_r - \beta \cdot \mathrm{sgn}(\nabla_x\ell(x_r, y_o)) \right)$
		\ENDFOR
		}
		\STATE \textit{// Do $K$-step PGD towards maximizing loss}
		\STATE $x_r^a = x_r + \eta$ where $\eta \sim B_{\epsilon}(0)$
		\FOR{$t=1, \ldots, K$}		
		\STATE $x_r^a \leftarrow \Pi_{x_r, \epsilon} \left( x_r^a + \alpha \cdot \mathrm{sgn}(\nabla_x\ell(x_r^a, y_o)) \right)$ 
		\ENDFOR
		\STATE $\theta \leftarrow \theta - \nabla_{\theta} \ell(x_r^a, y_o)$
		\ENDFOR
		\ENDFOR
	\end{algorithmic}
\end{algorithm}

\subsection{Preemptive Robustification for Classifiers with Randomized Smoothing}

Our preemptive robustification algorithm can also be applied to smoothed classifiers. Given a base classifier $c:\Xcal \to \Ycal$, a smoothed classifier $\tilde{c} : \Xcal \to \Ycal$ is defined as
\begin{align*}
	\tilde{c}(x) = \argmax_{y \in \Ycal} ~ \mathbb{P} \left( c(x + \xi) = y \right),
\end{align*}
where $\xi \sim \Ncal(0, \sigma^2 I)$. Crafting adversarial examples $x_r^a$ for the smoothed classifier, which is necessary for approximating the inner maximization in \Cref{eqn:safe_spot}, is ill-behaved since the $\argmax$ operation is non-differentiable. To address this problem, we follow the practice in \citet{salman19} and approximate the smoothed classifier $\tilde{c}$ with the smoothed soft classifier $\tilde{C}: \Xcal \to P(\Ycal)$ defined as
\begin{align}
	\tilde{C}(x) = \mathop{\mathbb{E}}_{\xi \sim \Ncal(0, \sigma^2 I)} \left[ C(x + \xi) \right],
	\label{eqn:smoothing}
\end{align}
where $P(\Ycal)$ is the set of probability distribution over $\Ycal$ and $C: \Xcal \to P(\Ycal)$ is the soft version of the base classifier $c$ such that $\argmax_{y \in \Ycal} C(x)_y = c(x)$. Finally, the adversarial example $x_r^a$ can be found by maximizing the cross-entropy loss of $\tilde{C}$ instead:
\begin{align*}
	& \maximize_{x_r^a} ~ -\log \left( \tilde{C}(x_r^a)_{c(x_o)} \right) \\
	&\: \mathrm{subject~to} ~~ \| x_r^a - x_r \|_p \le \epsilon,
\end{align*}
which can be approximated by $T$-step randomized PGD \citep{salman19}, where $\xi$ is sampled $M$ times to compute the sample mean of \Cref{eqn:smoothing} at each step. By replacing the inner maximization problem in \Cref{eqn:safe_spot} with the randomized PGD, we can update $x_r$ in a similar process.

\begin{algorithm}[t]
	\caption{Original image reconstruction}
	\label{alg:recon}
	\begin{algorithmic}[0]
		\INPUT Preemptively robustified image and its prediction $(x_r, c(x_r))$
		\STATE $\hat{x}_{o} = x_r$
		\FOR{$i=1, \ldots, \text{MAXITER}$}
		\STATE \textit{// Generate $N$ adversarial examples}
		\FOR{$n=1, \ldots, N$}
		\STATE $\hat{x}_{o, n}^a = \hat{x}_{o} + \eta$ where $\eta \sim \Ucal(B_\epsilon(0))$
		\FOR{$t=1, \ldots, T$}
		\STATE$ \hat{x}_{o, n}^a \leftarrow \Pi_{\hat{x}_{o}, \epsilon} \left( f(\hat{x}_{o, n}^a; c(x_r), \ell) \right)$
		\ENDFOR
		\ENDFOR
		\STATE \textit{// Update image} \\
		$\hat{x}_{o} \leftarrow \Pi_{x_r, \delta} \left( \hat{x}_{o}+ \beta \cdot \dfrac{1}{N}\sum\limits_{n=1}^N \dfrac{\partial \ell( \hat{x}_{o, n}^a, c(x_r))}{\partial \hat{x}_{o}} \right)$
		\ENDFOR
		\OUTPUT $\hat{x}_{o}$
	\end{algorithmic}
\end{algorithm}

\begin{algorithm}[t]
	\caption{Adaptive white-box attack}
	\label{alg:white_box}
	\begin{algorithmic}[0]
		\INPUT Preemptively robustified image and its prediction $(x_r, c(x_r))$, target $y_o$
		\STATE \textit{// Reconstruct original image} \\
		$\hat{x}_{o} = \textsc{OriginalImageReconstruction}(x_r, c(x_r))$
		\STATE \textit{// Run standard attack algorithm on reconstructed image} \\
		$\hat{x}_{o}^a = \textsc{AttackAlgorithm}(\hat{x}_{o}, y_o, \epsilon')$
		\OUTPUT $\hat{x}_{o}^a$
	\end{algorithmic}
\end{algorithm}

\subsection{Adaptive Attack against Preemptive Robustification}\label{sec:white_box}

So far, we have developed preemptive defense strategies against the grey-box adversary. Now, we consider the white-box adversary described in \Cref{sec:setup}, which is aware that the given image $x_r$ has been preemptively modified and aims to craft an adversarial example near the original image $x_o$ that is unknown. The most direct way for the adversary to achieve this is to reconstruct $x_o$ from $x_r$ and apply standard attack algorithms (\eg, PGD) on the reconstructed image $\hat{x}_{o}$. Since we assume the adversary knows the detailed hyperparameter settings of the robustification algorithm, the adversary can leverage this information to approximate the inverse dynamics of the preemptive robustification process starting from $x_r$, as described in \Cref{alg:recon}. The only difference between this reconstruction algorithm and the preemptive robustification process is the initialization and the update direction, modified to suit the reconstruction objective. \Cref{alg:white_box} shows the overall procedure of the adaptive white-box attack. Note that the adversary may modify $\hat{x}_{o}$ with a budget smaller than $\epsilon$ to induce $\hat{x}_{o}^a \in B_{\epsilon}(x_o)$, considering the original image might not be reconstructed accurately.

\Cref{fig:dist} shows the proposed reconstruction algorithm performs well in terms of reconstruction error if the preemptive robustification algorithm starts from the original image itself. However, as we run the preemptive robustification algorithm starting from a random point in $B_{\delta}(x_o)$, the performance of the reconstruction algorithm degrades considerably. About 80\% of the reconstructed images locate near the boundaries of $\epsilon$-balls centered at original images, which shows the difficulty of reconstructing the original image. We also observe that most of the resulting white-box attack examples generated from the reconstructed images lie outside $B_{\epsilon}(x_o)$, which implies they are not valid adversarial examples.

\begin{figure}[t]
	\centering
	\includegraphics[width=0.95\columnwidth]{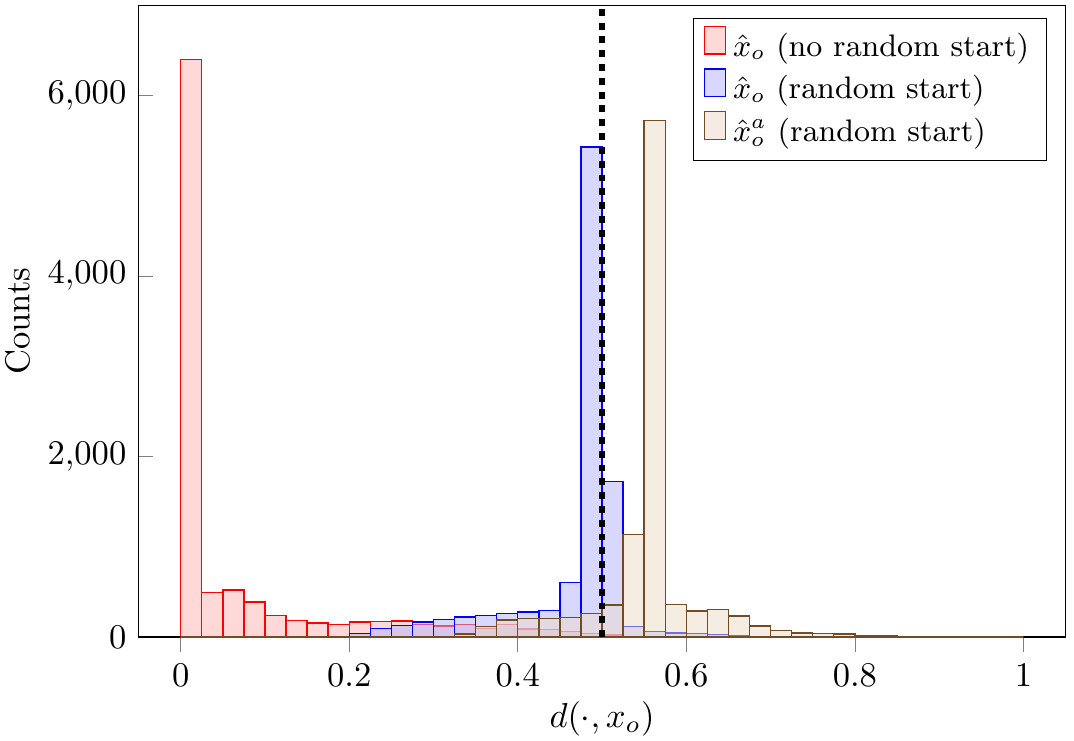}
	\caption{Histograms of the distances between original images $x_o$ and the reconstructed images $\hat{x}_{o}$ or their white-box attack examples $\hat{x}_{o}^a$ on CIFAR-10 test set. We use a preemptively robust model with $p=2$ and $\delta = \epsilon = 0.5$. The dotted line indicates the adversary's perturbation bound $\epsilon$.}
	\label{fig:dist}
\end{figure}

\section{Experiments}\label{sec:exp}

We evaluate our methods on CIFAR-10 and ImageNet by measuring classification accuracies of preemptively robustified images under the grey-box and white-box adversaries. As it is natural to assume that the defender and the adversary have the same modification budget, we set $\delta=\epsilon$ for all experiments. Both the adversaries use 20-step untargeted PGD and AutoAttack \cite{croce2020reliable} to find adversarial examples. For the white-box adversary, we sweep the final perturbation budget $\epsilon'$ and report the lowest accuracy measured. More details are listed in Supplementary C. The code is available online \footnote{\url{https://github.com/snu-mllab/preemptive_robustification}}.

\subsection{CIFAR-10}

We consider two types of perturbations: $\ell_\infty$ with $\epsilon=8/255$ and $\ell_2$ with $\epsilon=0.5$. To show the effectiveness of our robustification algorithm, we also report results without preemptive robustification (None)\footnote{In this case, the grey-box adversary is the same as the white-box adversary since no modification occurs on original images.}. As the baseline for our preemptively robust model, we use an adversarially trained model with early stopping \citep{rice20} (ADV).

The results in \Cref{tab:cifar_linf} and \Cref{tab:cifar_l2} show that our preemptive manipulation method successfully robustifies images compared to without the manipulation, achieving adversarial accuracies higher than 80\% in both the perturbation settings. Also, while the adversarially trained model does induce some images to be preemptively robust, our proposed network training method further boosts the performance of the robustified images in both the clean and adversarial accuracies.

Note that the white-box attacks are not effective on $\ell_\infty$ as they fail to lay their adversarial examples within the $B_\epsilon(x_o)$ ball. This is in part due to the characteristic of the $\ell_\infty$ distance measure, as $\ell_\infty$ distance can spike even when a single pixel deviates from the original value. On the other hand, white-box attacks on $\ell_2$ perturbations succeed to pose reasonable threats. However, the worst-case adversarial accuracy is still over 10\% higher than without our methods.

\begin{table}[ht]
	\centering
	\small
	\begin{tabular}{cc|ccccc}
		\toprule[1pt]
		\multirow{2}{*}{Model} & \multirow{2}{*}{Preempt.} & \multirow{2}{*}{Clean} & \multicolumn{2}{c}{Grey-box} & \multicolumn{2}{c}{White-box} \\
		& & & PGD & AA & PGD & AA \\
		\midrule
		ADV & None & 86.72 & 54.59 & 51.68 & 54.59 & 51.68 \\
		\midrule
		ADV & \textbf{Ours} & 86.72 & 86.23 & 81.70 & 86.72 & 86.72 \\
		\textbf{Ours} & \textbf{Ours} & \textbf{88.54} & \textbf{87.10} & \textbf{82.88} & \textbf{88.54} & \textbf{88.54} \\
		\bottomrule[1pt]
	\end{tabular}
	\caption{CIFAR-10 classification accuracy under grey-box and white-box adversaries with $\ell_\infty$ perturbation, $\epsilon=8/255$.}
	\label{tab:cifar_linf}
\end{table}

\begin{table}[ht]
	\centering
	\small
	\begin{tabular}{cc|ccccc}
		\toprule[1pt]
		\multirow{2}{*}{Model} & \multirow{2}{*}{Preempt.} & \multirow{2}{*}{Clean} & \multicolumn{2}{c}{Grey-box} & \multicolumn{2}{c}{White-box} \\
		& & & PGD & AA & PGD & AA \\
		\midrule
		ADV & None & 90.85 & 71.90 & 71.21 & 71.90 & 71.21  \\
		\midrule
		ADV & \textbf{Ours} & 90.85 & 84.81 & 83.56 & \textbf{85.12} & 79.48 \\
		\textbf{Ours} & \textbf{Ours} & \textbf{92.57} & \textbf{91.81} & \textbf{89.32} & 85.02 & \textbf{80.79} \\
		\bottomrule[1pt]
	\end{tabular}
	\caption{CIFAR-10 classification accuracy under grey-box and white-box adversaries with $\ell_2$ perturbation, $\epsilon=0.5$.}
	\label{tab:cifar_l2}
\end{table}

\subsection{ImageNet}\label{sec:imagenet}

We consider two types of perturbations: $\ell_\infty$ with $\epsilon=4/255$ and $\ell_2$ with $\epsilon=3.0$. As with the CIFAR-10 experiments, we compare our preemptive robustification algorithm to without the algorithm (None). As the baseline, we use a model adversarially trained with the fast training schemes \citep{fast} for computational efficiency (ADV).

\Cref{tab:imagenet_linf} and \Cref{tab:imagenet_l2} show our preemptive robustification methods scale to more practical datasets with bigger images. Our methods allow the natural images to be much more robust to adversarial attacks, with over 15\% higher worst-case adversarial accuracies, and at the same time maintains higher clean accuracies. We observe that similar to the CIFAR-10 experiments, the white-box attacks on the $\ell_\infty$ distance measure are not successful in finding appropriate adversarial samples.

\begin{table}[ht]
	\centering
	\small
	\begin{tabular}{cc|ccccc}
		\toprule[1pt]
		\multirow{2}{*}{Model} & \multirow{2}{*}{Preempt.} & \multirow{2}{*}{Clean} & \multicolumn{2}{c}{Grey-box} & \multicolumn{2}{c}{White-box} \\
		& & & PGD & AA & PGD & AA \\
		\midrule
		ADV & None & 56.24 & 32.03 & 27.52 & 32.03 & 27.52 \\
		\midrule
		ADV & \textbf{Ours} & 56.24 & 55.79 & 47.14 & 56.24 & 56.24 \\
		\textbf{Ours} & \textbf{Ours} & \textbf{61.01} & \textbf{59.66} & \textbf{48.24} & \textbf{61.01} & \textbf{61.01} \\
		\bottomrule[1pt]
	\end{tabular}
	\caption{ImageNet classification accuracy under grey-box and white-box adversaries with $\ell_\infty$ perturbation, $\epsilon=4/255$.}
	\label{tab:imagenet_linf}
\end{table}

\begin{table}[H]
	\centering
	\small	
	\begin{tabular}{cc|ccccc}
		\toprule[1pt]
		\multirow{2}{*}{Model} & \multirow{2}{*}{Preempt.} & \multirow{2}{*}{Clean} & \multicolumn{2}{c}{Grey-box} & \multicolumn{2}{c}{White-box} \\
		& & & PGD & AA & PGD & AA \\
		\midrule
		ADV & None & 54.99 & 32.07 & 27.58 & 32.07 & 27.58 \\
		\midrule
		ADV & \textbf{Ours} & 55.05 & 51.70 & 43.32 & 46.38 & 37.49 \\
		\textbf{Ours} & \textbf{Ours} & \textbf{61.60} & \textbf{58.13} & \textbf{43.60} & \textbf{54.23} & \textbf{47.54} \\
		\bottomrule[1pt]
	\end{tabular}
	\caption{ImageNet classification accuracy under grey-box and white-box adversaries with $\ell_2$ perturbation, $\epsilon=3.0$.}
	\label{tab:imagenet_l2}
\end{table}

\subsection{Randomized Smoothing}

We also evaluate our preemptive robustification algorithm for smoothed classifiers. We consider $\ell_2$ perturbations, where $\epsilon=0.5$ for CIFAR-10 and $\epsilon=3.0$ for ImageNet. We utilize a smoothed model trained with Gaussian noise augmentation as proposed in \citet{cohen19} due to its simplicity. We measure empirical adversarial accuracies using 20-step randomized PGD and its 10 restart version (PGD-10). We also compute the certified radii of the images and measure the certified adversarial accuracies.

\Cref{tab:smoothing_empirical_cifar} and \Cref{tab:smoothing_empirical_imagenet} shows the empirical robustness results against the randomized PGD. We observe our methods can significantly enhance preemptive robustness on the smoothed classifiers, maintaining 28\% and 49\% higher the worst-case adversarial accuracies than the baseline on CIFAR10 and ImageNet, respectively. The results in \Cref{tab:smoothing_certified} show our method also improves the certified robustness on the smoothed networks. Our methods achieve 22\% and 15\% higher certified accuracies on CIFAR10 and ImageNet, respectively.

\begin{table}[ht]
	\centering
	\small
	\begin{tabular}{c|ccccc}
		\toprule[1pt]
		\multirow{2}{*}{Preempt.} & \multirow{2}{*}{Clean} & \multicolumn{2}{c}{Grey-box} & \multicolumn{2}{c}{White-box} \\
		& & PGD & PGD-10 & PGD & PGD-10 \\ 
		\midrule
		None & 92.14 & 56.02 & 53.01 & 56.02 & 53.01 \\
		\textbf{Ours} & \textbf{92.35} & \textbf{91.37} & \textbf{89.98} & \textbf{82.06} & \textbf{80.71} \\
		\bottomrule[1pt]
	\end{tabular}
	\caption{CIFAR-10 empirical accuracy of smoothed network under grey-box and white-box adversaries with $\ell_2$ perturbation, $\epsilon=0.5$. We set the noise level to $\sigma=0.1$.}
	\label{tab:smoothing_empirical_cifar}
\end{table}

\begin{table}[H]
	\centering
	\small
	\begin{tabular}{c|ccccc}
		\toprule[1pt]
		 \multirow{2}{*}{Preempt.} & \multirow{2}{*}{Clean} & \multicolumn{2}{c}{Grey-box} & \multicolumn{2}{c}{White-box} \\
		& & PGD & PGD-10 & PGD & PGD-10 \\ 
		\midrule
		None & 69.93 & 9.61 & 8.61 & 9.61 & 8.61 \\
		\textbf{Ours} & \textbf{70.05} & \textbf{62.27} & \textbf{57.24} & \textbf{68.05} & \textbf{67.72} \\
		\bottomrule[1pt]
	\end{tabular}
	\caption{ImageNet empirical accuracy of smoothed network under grey-box and white-box adversaries with $\ell_2$ perturbation, $\epsilon=3.0$. We set the noise level to $\sigma=0.25$.}
	\label{tab:smoothing_empirical_imagenet}
\end{table}

\begin{table}[H]
	\centering
	\small
	\begin{tabular}{c|cc}
		\toprule[1pt]
		Preempt. & Clean & Cert. \\
		\midrule
		None & 82.84 & 55.58 \\
		\textbf{Ours} & \textbf{84.72} & \textbf{77.95} \\
		\bottomrule[1pt]
	\end{tabular}
	\hfil
	\begin{tabular}{c|cc}
		\toprule[1pt]
		Preempt. & Clean & Cert. \\
		\midrule
		None & 47.02 & 12.68 \\
		\textbf{Ours} & \textbf{52.66} & \textbf{27.89} \\
		\bottomrule[1pt]
	\end{tabular}
	\caption{CIFAR-10 (left) and ImageNet (right) certified accuracies of smoothed network with $\ell_2$ perturbation. The noise levels are $\sigma=0.25$ for CIFAR-10 and $\sigma=1.0$ for ImageNet.}
	\label{tab:smoothing_certified}
\end{table}

\section{Conclusion}\label{sec:conclusion}

We consider a real-world adversarial framework where the MitM adversary intercepts and manipulates the images during transmission. To protect users from such attacks, we introduce a novel optimization algorithm for finding robust points in the vicinity of original images along with a new network training method suited for enhancing preemptive robustness. The experiments show that our algorithm can find such robust points for most of the correctly classified images. Further results show our method also improves preemptive robustness on smooth classifiers.

\section*{Acknowledgement}

This work was supported in part by SNU-NAVER Hyperscale AI Center and Institute of Information \& Communications Technology Planning \& Evaluation (IITP) grant funded by the Korea government (MSIT) (No. 2020-0-00882, (SW STAR LAB) Development of deployable learning intelligence via self-sustainable and trustworthy machine learning and No. 2019-0-01371, Development of brain-inspired AI with human-like intelligence). This material is based upon work supported by the Air Force Office of Scientific Research under award number FA2386-20-1-4043. Hyun Oh Song is the corresponding author.

\bibliography{aaai22}

\clearpage

\appendix

\section{Preemptive Robustification Algorithm}

\subsection{Proof of Lemma 1}
Since $x_r^a \in B_{\epsilon}(x_r)$, we have 
\begin{align}
	\label{eqn:ineq}
	\ell(x_r, c(x_o)) \le \sup_{x_r^a} ~ \ell(x_r^a, c(x_o)) = \tilde{h}(x) \le -\log(0.5).
\end{align}
Let $C(x_r)$ be the softmax probability of $x_r$. \Cref{eqn:ineq} implies $C(x_r)_{c(x_o)} \ge 0.5$, \ie, $c(x_r) = c(x_o)$.
Finally, we have
\begin{align*}
	h(x_r) &= \ell(x_r, c(x_o)) + \sup_{x_r^a} ~ \ell(x_r^a, c(x_r)) \\
	&= \ell(x_r, c(x_o)) + \sup_{x_r^a} ~ \ell(x_r^a, c(x_o)) \tag{$\because c(x_r) = c(x_o)$} \\
	&\le 2 \sup_{x_r^a} ~ \ell(x_r^a, c(x_o)) \tag{$\because \ell(x_r, c(x_o)) \le \sup_{x_r^a} \ell(x_r^a, c(x_o))$} \\
	&= 2\tilde{h}(x_r).
\end{align*}
\qed

\subsection{Satisfiability of the Assumption in Lemma 1}

Lemma 1 requires the assumption that the loss $\tilde{h}(x_r)= \sup_{x_r^a \in B_\epsilon(x_r)} \ell(x_r^a, c(x_o))$ should be less or equal to $-\log(0.5)$. Our preemptive robustifcation algorithm can sufficiently minimize the loss to meet the assumption. To verify this, we measure the loss values of the resulting preemptively robustified images $x_r$ from CIFAR-10 test images that are correctly classified. \Cref{fig:xents_safe_adv} shows that 99.3\% of the images satisfy the assumption.

\section{Computing Update Gradient without Second-Order Derivatives}

\subsection{Proof of Lemma 2}

It is enough to compute the Jacobian of \small$\dfrac{g}{\|g\|_2}$\normalsize. By the quotient rule, we have
\begin{align}
	\label{eqn:quotient}
	\nabla_x \left( \frac{g}{\|g\|_2} \right) &= \frac{\|g\|_2 \cdot H - g \cdot (\nabla_x \| g \|_2)^\intercal }{\|g\|_2^2} \nonumber \\
	&= \frac{H}{\| g \|_2} - \frac{g \cdot (\nabla_x \| g \|_2)^\intercal}{\|g\|_2^2}.
\end{align}    
Now, we compute $\nabla_x \| g \|_2$. Since $\| g \|_2^2 = \langle g, g \rangle$, we have
\begin{align}
	\label{eqn:inner}
	2 \| g \|_2 \cdot \nabla_x \| g \|_2 &= \nabla_x \langle g, g \rangle = 2 H \cdot g.
\end{align}    
Plugging \Cref{eqn:inner} into \Cref{eqn:quotient}, we have
\begin{align*}
	\nabla_x \left( \frac{g}{\|g\|_2} \right) &= \frac{H}{\| g \|_2} - \frac{g \cdot (\nabla_x \| g \|_2)^\intercal}{\|g\|_2^2} \\
	&= \frac{H}{\| g \|_2} - \frac{g \cdot g^\intercal \cdot H^\intercal }{\|g\|_2^3} \\
	&= \left(I - \left( \frac{g}{\| g \|_2} \right) \left (\frac{g}{\| g \|_2} \right)^\intercal \right) \frac{H}{\| g \|_2},
\end{align*}
since $H$ is symmetric, and this completes the proof. \qed

\begin{figure}[ht]
	\centering
	\includegraphics[width=\columnwidth]{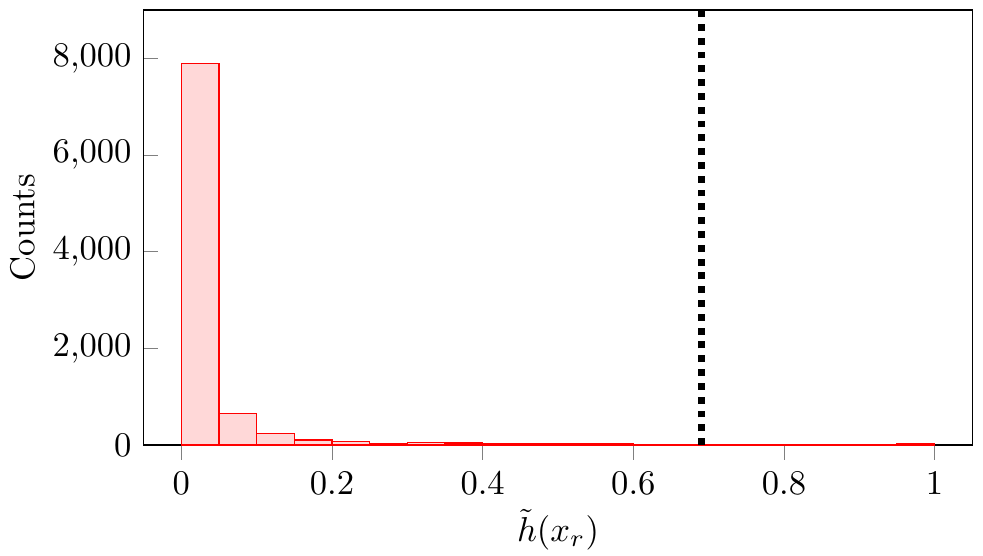}
	\caption{Histogram of the loss values $\tilde{h}(x_r)$ on CIFAR-10 test images. The dotted line indicate the upper bound for $\tilde{h}(x_r)$ to satisfy the assumption in Lemma 1.}
	\label{fig:xents_safe_adv}
\end{figure}

\subsection{Proof of Proposition 1}

Note that \small$P = I - \left( \dfrac{g}{\| g \|_2} \right) \left( \dfrac{g}{\| g \|_2} \right)^\intercal$\normalsize is a projection map onto a hyperplane whose normal vector is \small$\dfrac{g}{\| g \|_2}$\normalsize. Since a projection map is a contraction map, we have
\small
\begin{align*}
	& \left\| \nabla_x f^\intercal a \right\|_2 \\
	&~~= \left\| \left( I + \alpha \cdot \left(I - \left(\frac{g}{\| g \|_2}\right)\left(\frac{g}{\| g \|_2}\right)^\intercal \right) \frac{H}{\| g \|_2} \right)^\intercal a \right\|_2 \tag{$\because$ By Lemma 2} \\
	&~~= \left\| a + \alpha \cdot \frac{H^\intercal}{\| g \|_2} \left(I - \left(\frac{g}{\| g \|_2}\right)\left(\frac{g}{\| g \|_2}\right)^\intercal \right)^\intercal a \right\|_2 \\
	&~~= \left\| a \right\|_2 + \left\| \alpha \cdot \frac{H^\intercal}{\| g \|_2} \left(I - \left(\frac{g}{\| g \|_2}\right)\left(\frac{g}{\| g \|_2}\right)^\intercal \right)^\intercal a \right\|_2 \tag{$\because$ By the triangular inequality} \\
	&~~\le \left\| a \right\|_2 + \alpha \cdot \frac{\left\| H^\intercal \right\|_2}{\| g \|_2} \cdot \left\| \left(I - \left(\frac{g}{\| g \|_2}\right)\left(\frac{g}{\| g \|_2}\right)^\intercal \right)^\intercal \right\|_2 \left\| a \right\|_2 \tag{$\because \| A B \|_2 \le \| A \|_2 \| B \|_2$} \\
	&~~= \left\| a \right\|_2 + \alpha \cdot \frac{\sigma}{\| g \|_2} \cdot \left\| \left(I - \left(\frac{g}{\| g \|_2}\right)\left(\frac{g}{\| g \|_2}\right)^\intercal \right)^\intercal \right\|_2 \left\| a \right\|_2 \tag{$\because$ $H$ is symmetric}\\
	&~~= \left\| a \right\|_2 + \alpha \cdot \frac{\sigma}{\| g \|_2} \cdot \left\| a \right\|_2 \tag{$\because$ $P$ is a contraction map},
\end{align*}
\normalsize

Finally, the projection operator $\Pi_{x_r, \epsilon}$ has a form of $\Pi_{x_r, \epsilon}(x) = k(x - x_r) + x_r$, where  $0 < k \le 1$. Therefore, we have 
\begin{align*}
	 \left\| \nabla_x \tilde{f}^\intercal a \right\|_2 &= \left\| \nabla_x f^\intercal \nabla_x  \left( \Pi_{x_r, \epsilon} \right)^\intercal a \right\|_2 \tag{$\because ~ \tilde{f} = \Pi_{x_r, \epsilon} \circ f$ } \\
	 &= k \left\| \nabla_x f^\intercal a \right\|_2 \tag{$\because~ \nabla_x \Pi_{x_r, \epsilon} = kI $} \\
	& \le \left\| \nabla_x f^\intercal a \right\|_2 \tag{$\because ~ 0 < k \le 1$} \\
	& \le \left\| a \right\|_2 + \alpha \cdot \frac{\sigma}{\| g \|_2} \cdot \left\| a \right\|_2,
\end{align*}
which completes the proof. \qed

\begin{figure}[ht]
	\centering
	\includegraphics[width=0.9\columnwidth]{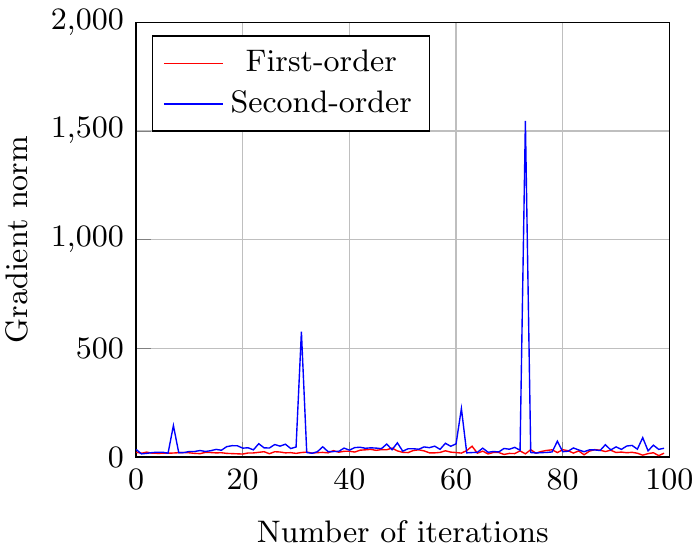}
	\caption{The $\ell_2$ norm of update gradient at each iteration during preemptive robustification process with the first-order approximation and with the exact computation using second-order derivatives for an image in CIFAR-10 test set. The approximate gradient update can successfully robustify the image, while the exact gradient update does not.}
	\label{fig:norm_grad}
\end{figure}

\subsection{Experiments with Second-Order Derivatives}

To examine whether the exact gradient computation incurs exploding update gradients, we measure the $\ell_2$ norm of the update gradient computed with second-order derivatives at each step, compared to its first-order approximation. We run our experiment on a CIFAR-10 classifier and consider $\ell_2$ perturbations under $\delta = \epsilon = 0.5$. \Cref{fig:norm_grad} shows that preemptive robustification with exact gradient computation leads to unstable update gradients, in contrast to the approximate gradient update.

\section{Implementation Details}\label{sec:impl}

We fix $\text{MAXITER} =100$, $T=20$, $\alpha=\epsilon / 4$, $N=1$, and $L=1$ throughout all the experiments. For implementing the preemptive robustification algorithm, we transform $B_\delta(x_o)$ to $\mathbb{R}^n$ via $\mathrm{tanh}$ transformation and use RMSProp optimizer \citep{tieleman2012lecture} to update preemptively robustified images when dealing with $\ell_\infty$ perturbations. Under $\ell_2$ perturbations, we use projected gradient descent method as default. For implementing the adaptive white-box attack, we sweep the final perturbation budget of the adversary with $\epsilon' \in \{ 0.25\epsilon, 0.5\epsilon, 0.75\epsilon, 1.0\epsilon \}$.

\subsection{Experiments on CIFAR-10 and ImageNet}

\paragraph{CIFAR-10} For network training, we use Wide ResNet architecture \cite{zagoruyko2017wide} with depth 34 and width 10. We train the models for 200 epochs with an initial learning rate of $0.1$, decayed with a factor of $0.1$ on epoch 100 and 150. We use SGD optimizer with weight decay $5E-4$, momentum $0.9$, and batch size $128$. We trained ADV models \citep{rice20} using 10-step PGD with a step size of $\epsilon / 4$ and select the checkpoint with the best performance on the test set. For preemptively robust models, we use the same 10-step PGD for adversarial example generations, set $\beta$ during model training to $\epsilon$, and choose the latest checkpoint.

We take the whole 10,000 images from CIFAR-10 test set and run our preemptive robustification algorithm on them. On $\ell_\infty$ perturbations with $\epsilon=8/255$, we set the learning rate $\beta=0.1$, and on $\ell_2$ perturbations with $\epsilon=0.5$, we set $\beta=0.001$. Each $\beta$ is chosen by tuning between a range of $\{1.0, 0.5, 0.1, 0.05, 0.01, 0.001\}$ and $\{0.1, 0.01, 0.001, 0.0001, 0.00001\}$. 

\paragraph{ImageNet} For network training, we use ResNet-50 architecture \cite{he2016}. For ADV models, we train the models with adversarial training using 2-step PGD to reduce the training cost. The PGD step size is set to $\epsilon$, following the protocol of \citet{fast}. Additionally, we adopt the learning rate and input preprocessing schedule of \citet{fast}, except that we set the input image size to $224 \times 224$ in phase 3, instead of $288 \times 288$. For preemptively robust models, we use the 2-step PGD for adversarial example generation and tune $\delta$ and $\beta$ during model training to $\epsilon/2$.

We randomly sample 10,000 images from the test set, randomly crop and resize them to $224 \times 224$, and run our preemptive robustification algorithm on the transformed images. We consider $\ell_\infty$ perturbations with $\epsilon=4/255$ and $\ell_2$ perturbations with $\epsilon=3.0$. On the $\ell_\infty$ perturbations, $\beta$ is set to $0.1$, chosen by tuning between a range of $\{1.0, 0.5, 0.1, 0.05, 0.01\}$. On $\ell_2$ perturbations, we tuned $\beta$ between a range of $\{1.0, 0.1, 0.01, 0.001, 0.0001\}$ and choose 0.1.

\subsection{Experiments on Randomized Smoothing}

We use the same network structures from Section 4.1 and 4.2 and train networks with additive Gaussian noises drawn from $\mathcal{N}(0, \sigma^2 I)$. For smoothed models on ImageNet, we use the pre-trained models provided by \citet{cohen19}. We take 50 Gaussian samples and choose the most probable class by a majority vote for class prediction.

To evaluate the empirical robustness on the smoothed classifiers, we set the noise level $\sigma$ to 0.1 on CIFAR-10 and 0.25 on ImageNet. The PGD settings are the same as in the experiments for base classifiers. We take 5 Gaussian samples for preemptive robustification and 50 Gaussian samples for evaluation.

To measure the certified robustness on the smoothed classifiers, we set noise levels to be larger than those used in evaluating the empirical robustness since a larger $\sigma$ leads to larger certified radii. We set $\sigma$ to 0.25 and 1.0 for the smoothed classifiers on CIFAR-10 and ImageNet, respectively. We take 100,000 Gaussian samples to compute certified radii of images.

For CIFAR-10, the learning rate $\beta$ is set to $0.005$, chosen within a range of $\{0.01, 0.005, 0.001, 0.0005, 0.0001\}$. For ImageNet, the learning rate $\beta$ is set to $0.1$, chosen within a range of $\{1.0, 0.5, 0.1, 0.05, 0.01\}$.

\subsection{Computing Infrastructure}

We conduct all experiments on Ubuntu 16.04 with a hardware system comprised of Intel Xeon Gold 5220R CPU, 256GB RAM, and NVIDIA RTX 2080 Ti GPU. We use PyTorch \cite{pytorch} as a deep learning framework.

\section{Visualization of Preemptively Robustified Examples}

To understand how our preemptive robustification algorithm modifies each image, we tried increasing both the modification budgets ($\delta$ and $\epsilon$) for visualization. \Cref{fig:examples} shows the resulting examples. Interestingly, we find that our robustification process emphasizes human perceptible features of the original images, such as mountain peaks covered with snow or the fur color of the fox. Also, it tends to render colors much more vivid, and the object boundaries crisper. The results suggest that our preemptive robustification algorithm could also be applied to image synthesis tasks such as style transfer, super-resolution, or colorization, aligned with the findings of \citet{santurkar19single}.

\begin{figure*}[ht]
	\centering
	\includegraphics[width=\textwidth]{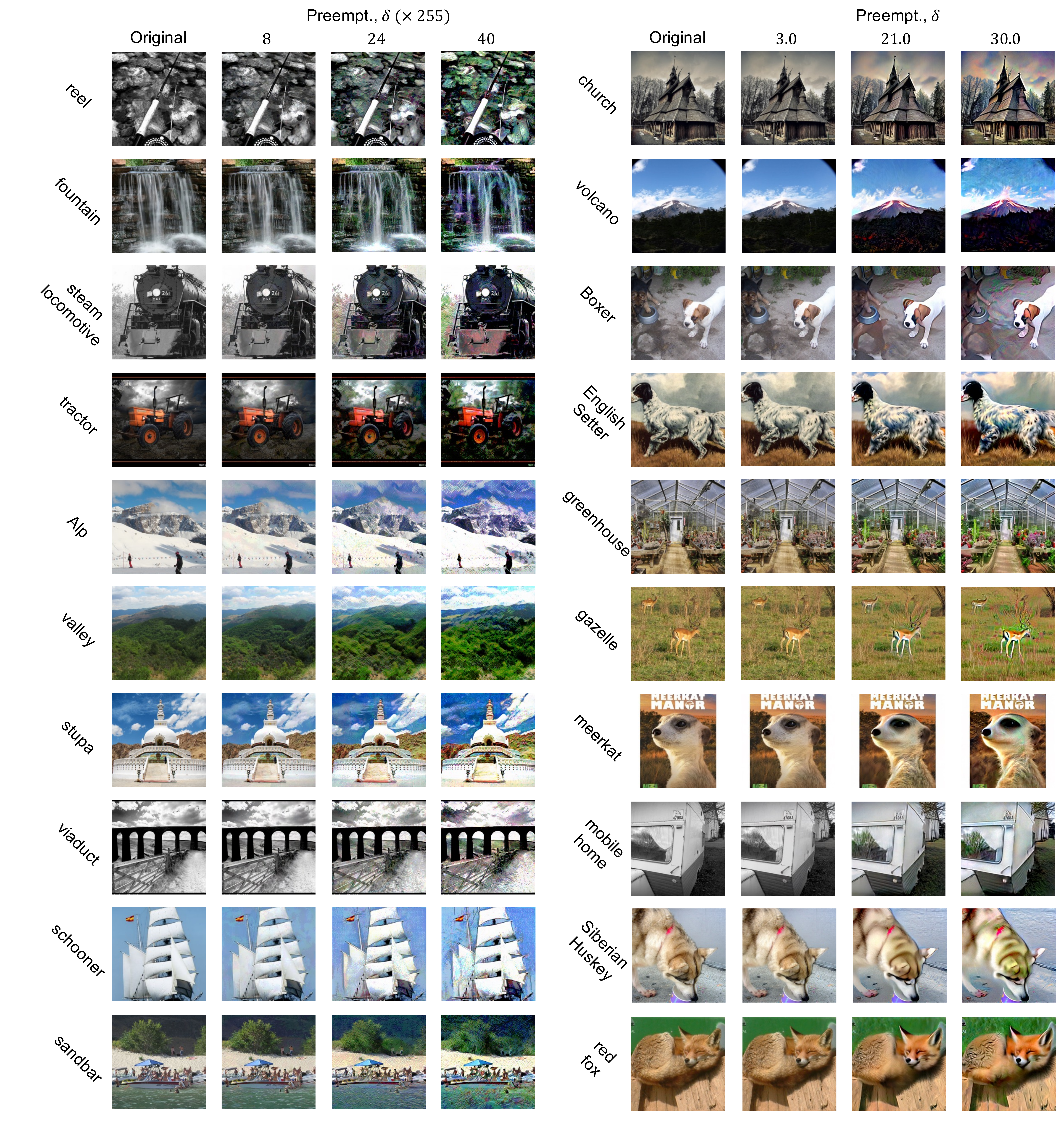}
	\caption{Examples of preemptively robustified images on ImageNet in $\ell_\infty$ (left) and $\ell_2$ (right) settings.}
	\label{fig:examples}
\end{figure*}

\end{document}